\documentclass[sigconf]{acmart}

\usepackage{booktabs} % For formal tables
\usepackage{amsmath} % For math
\usepackage{mathtools} % For math
\usepackage{multirow}
\usepackage{flushend}
\usepackage{graphicx}
    \setkeys{Gin}{width=5.5cm}

\graphicspath{ {images/} }

\theoremstyle{definition}

\copyrightyear{2017} 
\acmYear{2017} 
\setcopyright{acmlicensed}
\acmConference{KDD '17}{August 13-17, 2017}{Halifax, NS, Canada}\acmPrice{15.00}\acmDOI{10.1145/3097983.3098155}
\acmISBN{978-1-4503-4887-4/17/08}

\begin{document}

\title{Effective Evaluation Using Logged Bandit Feedback from Multiple Loggers}

\author{Aman Agarwal, Soumya Basu, Tobias Schnabel, Thorsten Joachims}
\affiliation{%
  \institution{Cornell University, Dept. of Computer Science}
  \city{Ithaca} 
  \state{NY}
  \country{USA}}
\email{[aa2398,sb2352,tbs49,tj36]@cornell.edu}

%\renewcommand{\shortauthors}{Author names omitted}

\begin{abstract}

Accurately evaluating new policies (e.g. ad-placement models, ranking functions, recommendation functions) is one of the key prerequisites for improving interactive systems. While the conventional approach to evaluation relies on online A/B tests, recent work has shown that counterfactual estimators can provide an inexpensive and fast alternative, since they can be applied offline using log data that was collected from a different policy fielded in the past. In this paper, we address the question of how to estimate the performance of a new target policy when we have log data from multiple historic policies. This question is of great relevance in practice, since policies get updated frequently in most online systems. 
We show that naively combining data from multiple logging policies can be highly suboptimal. In particular, we find that the standard Inverse Propensity Score (IPS) estimator suffers especially when logging and target policies diverge -- to a point where throwing away data improves the variance of the estimator. We therefore propose two alternative estimators which we characterize theoretically and compare experimentally. We find that the new estimators can provide substantially improved estimation accuracy.

\end{abstract}



% We no longer use \terms command


\maketitle

%A category including the fourth, optional field follows...
%\category{D.2.8}{Software Engineering}{Metrics}[complexity measures, performance measures]


%\keywords{ACM proceedings, \LaTeX, text tagging} % NOT required for Proceedings

\section{Introduction}
Interactive systems (e.g., search engines, ad-placement systems, recommender systems, e-commerce sites) are typically evaluated according to online metrics (e.g., click through rates, dwell times) that reflect the users' response to the actions taken by the system. For this reason, A/B tests are of widespread use in which the new policy to be evaluated is fielded to a subsample of the user population. Unfortunately, A/B tests come with two drawbacks. First, they can be detrimental to the user experience if the new policy to be evaluated performs poorly. Second, the number of new policies that can be evaluated in a given amount of time is limited, simply because each A/B test needs to be run on a certain fraction of the overall traffic and should ideally span any cycles (e.g. weekly patterns) in user behavior.

Recent work on counterfactual evaluation techniques provides a principled alternative to A/B tests that does not have these drawbacks \cite{li2011unbiased,li2015counterfactual,bottou2013counterfactual,swaminathan2015counterfactual}. These techniques do not require that the new policy be deployed online, but they instead allow reusing logged interaction data that was collected by a different policy in the past. In this way, these estimators address the counterfactual inference question of how a new policy would have performed, if it had been deployed instead of the old policy that actually logged the data. This allows reusing the same logged data for evaluating many new policies, greatly improving scalability and timeliness compared to A/B tests. 

In this paper, we address the problem of counterfactual evaluation when log data is available not just from one logging policy, but from multiple logging policies. Having data from multiple policies is common to most practical settings where systems are repeatedly modified and deployed. While the standard counterfactual estimators based on inverse propensity scores (IPS) apply to this situation, we show that they are suboptimal in terms of their estimation quality. In particular, we investigate the common setting where the log data takes the form of contextual bandit feedback from a stochastic policy, showing that the variance of the conventional IPS estimator suffers substantially when the historic policies are sufficiently different -- to a point where throwing away data improves the variance of the estimator. To overcome the statistical inefficiency of the conventional IPS estimator, we explore two alternative estimators that directly account for the data coming from multiple different logging policies. We show theoretically that both estimators are unbiased, and have lower variance than the conventional IPS estimator. Furthermore, we quantify the amount of variance reduction in an extensive empirical evaluation that demonstrates the effectiveness of both the estimators.

\section{Related Work}
The problem of re-using logged bandit feedback is often part of counterfactual learning \cite{bottou2013counterfactual,li2015counterfactual,swaminathan2015counterfactual}, and more generally can be viewed as part of off-policy evaluation in reinforcement learning \cite{sutton1998reinforcement,precup2000eligibility}.

In counterfactual learning, solving the evaluation problem is often the first step to deriving a learning algorithm \cite{strehl2010learning,bottou2013counterfactual,swaminathan2015counterfactual}. The key to being able to counterfactually reason based on logged data is randomness in the logged data. Approaches differ in how randomness is being included in the policies. For example, in \cite{li2015counterfactual} randomization is directly applied to the actions of each policy, whereas \cite{bottou2013counterfactual} randomizes individual policy parameters to create a distribution over actions. 

In exploration scavenging \cite{langford2008exploration}, the authors address counterfactual evaluation in a setting where the actions do not depend on the context. They mention the possibility of combining data from different policies by interpreting each policy as an action. Li et al. \cite{li2015toward} propose to use naturally occurring randomness in the logged data when policies change due to system changes. Since this natural randomness may not be entirely under the operator's control, the authors propose to estimate the probability that a certain logging policy was in place to recover propensities. %However, no further experiments are provided to validate their method. 
The balanced IPS estimator studied in this paper could serve as a starting point for further techniques in that direction. 

Evaluation from logged data has often been studied with respect to specific domains, for example in news recommendation \cite{li2010contextual,li2011unbiased,li2015counterfactual} as well as in information retrieval \cite{hofmann2013reusing,li2015counterfactual}. The work by Li et al. \cite{li2011unbiased} highlights another common use-case in practice, where different logging policies are all active at the same time, focusing on the evaluation of different new methods.
The estimators in this paper can naturally be applied to this scenario as well to augment logging data of one policy with the data from others. An interesting example for probabilistic policies can be found in  \cite{hofmann2013reusing}, where the authors consider policies that are the probabilistic interleaving of two deterministic ranking policies and use log data to pre-select new candidate policies. 

Very related to combining logs from different policies is the problem of combining samples coming
from different proposal distributions in importance sampling \cite{owen2000safe,mcbook,elvira2015efficient}.
There, samples are drawn from multiple proposal distributions and need to be combined in a way that
reduces variance of the combined estimator. Multiple importance sampling has been particularly studied in computer graphics \cite{veach1995optimally}, as Monte Carlo techniques are employed for rendering.
Most related to the weighted IPS estimator presented later in the paper is adaptive multiple importance sampling (AMIS) \cite{cornuet2012adaptive,elvira2015generalized} that also recognizes that it is not optimal to weigh contributions from all proposal distributions the same, but instead updates weights as well as the proposal distributions after each sampling step. The most notable differences to our setting here are that (i) we regard the sampling distributions as given and fixed, and (ii) the sampled log data is also fixed.
An interesting avenue for future work would be to use control variates to further reduce variance of our estimators \cite{owen2000safe,heraOptimal}, although this approach is computationally demanding since it requires solving a quadratic problem to determine optimal
weights.

Another related area is sampling-based evaluation of information retrieval systems \cite{yilmaz2008simple,carterette2009,schnabel2016comparative}. Instead of feedback data that stems from interactions with users, the observed feedback comes from judges. A policy in this case corresponds to a sampling strategy which determines the query-document pairs to be sent out for judgement. As shown by Carterette et al. \cite{carterette2009}, relying on sampling-based elicitation schemes cuts down the number of required judgements substantially as compared to a classic deterministic pooling scheme. The techniques proposed in our paper could also be applied to the evaluation of retrieval systems when data from different judgement pools need to be combined.

\section{Problem Setting}
\label{sec:problem}

\newcommand{\Xset}{\ensuremath{{\mathcal{X}}}}
\newcommand{\Yset}{\ensuremath{{\mathcal{Y}}}}
\newcommand{\h}{\ensuremath{{\pi}}}
\newcommand{\htar}{\ensuremath{{\bar{\h}}}}
\newcommand{\hlog}[1]{\ensuremath{{\h_{#1}}}}
\newcommand{\havg}{\ensuremath{{\h_{avg}}}}
\newcommand{\Utrue}[1]{\ensuremath{{U({#1})}}}
\newcommand{\D}[1]{\ensuremath{{\mathcal{D}^{#1}}}}
\newcommand{\Dall}{\ensuremath{{\mathcal{D}}}}
\newcommand{\x}[2]{\ensuremath{{x^{#1}_{#2}}}}
\newcommand{\y}[2]{\ensuremath{{y^{#1}_{#2}}}}
\newcommand{\p}[2]{\ensuremath{{p^{#1}_{#2}}}}
\newcommand{\del}[2]{\ensuremath{{\delta^{#1}_{#2}}}}
\newcommand{\num}{\ensuremath{{m}}}
\newcommand{\size}[1]{\ensuremath{{n_{#1}}}}
\newcommand{\sizeall}{\ensuremath{{n}}}
\newcommand{\Unaive}{\ensuremath{{\hat{U}_{naive}(\htar)}}}
\newcommand{\Ubal}{\ensuremath{{\hat{U}_{bal}(\htar)}}}
\newcommand{\Uwt}{\ensuremath{{\hat{U}_{weight}(\htar)}}}
\newcommand{\Ugen}{\ensuremath{{\hat{U}(\htar)}}}
\newcommand{\Ugenweight}{\ensuremath{{\hat{U}_\lambda(\htar)}}}
\newcommand{\wt}[1]{\ensuremath{{\lambda_{#1}}}}
\newcommand{\wtopt}[1]{\ensuremath{{\lambda^{*}_{#1}}}}
\newcommand{\diverg}[2]{\ensuremath{{\sigma_\delta^2({#1}||{#2)}}}}
\newcommand{\diverghat}[2]{\ensuremath{{\hat{\sigma}_\delta^2({#1}||{#2)}}}}
\newcommand{\var}[1]{\ensuremath{{\mathrm{Var}_{\Dall}[{#1}]}}}

In this paper, we study the use of logged Bandit feedback that arises in interactive learning systems. In these systems, the system receives as input a vector $x \in \Xset$, typically encoding user input or other contextual information.
Based on input $x$, the system responds with an action $y \in \Yset$ for which it receives some feedback in the form of a cardinal utility value $\delta: \Xset \times \Yset \mapsto \mathbb{R}$. Since the system only receives feedback for the action $y$ that it actually takes, this feedback is often referred to as Bandit feedback \cite{swaminathan2015counterfactual}.

For example, in ad placement models, the input $x$ typically encodes user-specific information as well as the web page content, and the system responds with an ad $y$ which is then displayed on the page. Finally, user feedback $\delta(x,y)$ for the displayed ad is presented, such as whether the ad was clicked or not. Similarly, for a news website, the input $x$ may encode user-specific and other contextual information to which the system responds with a personalized home page $y$. In this setting, the user feedback $\delta(x,y)$ could be the time spent by the user on the news website. 

%For example, in search engines the input $x$ typically encodes the query as well as user-specific information, and the system responds with a ranked list of search results $y$. Finally, the search engine observes user feedback $\delta(x,y)$ on the ranking it presented, such as the number of clicks or the time spent.

In order to be able to counterfactually evaluate new policies, we consider \textit{stochastic policies} $\h$ that define a probability distribution over the output space $\Yset$. Predictions are made by sampling $y \sim \h(\Yset|x)$ from a policy given input $x$.  The inputs are assumed to be drawn i.i.d.\ from a fixed but unknown distribution $x \stackrel{i.i.d.}{\sim} Pr(\Xset)$. The feedback $\delta(x,y)$ is a cardinal utility that is only observed at the sampled data points. Large values for $\delta(x,y)$ indicate user satisfaction with $y$ for $x$, while small values indicate dissatisfaction. 
%For notational convenience, denote the probability distribution $\h(\Yset|x)$ induced by $x$ over the output space by $\h(x)$, and the probability assigned by $\h(x)$ to $y$ as $\h(y|x)$.

We evaluate and compare different policies with respect to their induced utilities. The utility of a policy $\Utrue{\h}$ is defined as the expected utility of its predictions under both the input distribution as well as the stochastic policy. More formally:

\begin{definition}[Utility of Policy] The utility of a policy $\h$ is
\begin{align*}
  \Utrue{\h} &\equiv \mathbb{E}_{x \sim \Pr(\Xset)} \mathbb{E}_{y \sim \h(\Yset|x)}[\delta(x, y)] \\
   &= \sum_{\mathclap{\substack{x \in \Xset \\ y \in \Yset}}}
\text{Pr}(x) \h(y | x) \delta(x, y)
\end{align*}
\end{definition}

Our goal is to re-use the interaction logs collected from multiple historic policies to estimate the utility of a new policy. In this paper, we denote the  the new policy (also called the target policy) as $\htar$, and the $\num$ logging policies as $\hlog{1}, \dots, \hlog{n}$. The log data collected from each logging policy $\hlog{i}$
is 
\begin{align*}
    \D{i} = \{(\x{i}{1}, \y{i}{1}, \del{i}{1}, \p{i}{1}), \dots, (\x{i}{\size{i}}, \y{i}{\size{i}}, \del{i}{\size{i}}, \p{i}{\size{i}})\},
\end{align*} 
where $\size{i}$ data-points are collected from logging policy $\hlog{i}$, $\x{i}{j} \sim \Pr(\Xset)$, $\y{i}{j} \sim \hlog{i}(\Yset|\x{i}{j})$, $\del{i}{j} \equiv
\delta(\x{i}{j},\y{i}{j})$, and $\p{i}{j} \equiv \hlog{i}(\y{i}{j} | \x{i}{j})$. Note that during the operation of the
logging policies, the propensities $\hlog{i}(y|x)$ are tracked and appended to the logs. We will also
assume that the quantity $\hlog{i}(y|x)$ is available at all $(x,y)$ pairs. This is a very mild
assumption since the logging policies were designed and controlled by us, so their code can be stored.
Finally, let $\Dall = \bigcup\limits_{i=1}^{\num} \D{i}$ denote the combined collection of log data over all the logging policies, and $\sizeall = \sum_{i=1}^\num \size{i}$ denote the total number of samples. 

%TJ: The example was somewhat out of place, and it refers to the toy example before it was introduced. I would advocate skipping it.
%{\bf Example.}
%Assume we want to evaluate the utility of the target policy $\htar$ given in Table~\ref{tab:toy}. According to our definition, 
%\begin{align*}
%  \Utrue{\htar} &= \sum_{\mathclap{\substack{x \in \Xset \\ y \in \Yset}}} \text{Pr}(x) \htar(y | x) \delta(x, y) \\
%&= \sum_{\mathclap{\substack{x \in \{x_1, x_2\} \\ y \in \{y_1, y_2\}}}} \text{Pr}(x) \htar(y | x) \delta(x, y) \\
%&= 0.5\cdot0.8\cdot10 + 0.5\cdot0.2\cdot1 + 0.5\cdot0.2\cdot1 + 0.5\cdot0.8\cdot10  \\
%&= 8.2.
%\end{align*}
Unfortunately, it is not possible to directly compute the utility of a policy based on log data using the formula from the definition above. While we have a random sample of the contexts $x$ and the target policy $\h(y | x)$ is known by construction, we lack full information about the feedback $\delta(x,y)$. In particular, we know $\delta(x,y)$ only for the particular action chosen by the logging policy, but we do not necessarily know it for all the actions that the target policy $\h(y | x)$ can choose. In short, we only have logged bandit feedback, but not full-information feedback. This motivates the use of statistical estimators to overcome the infeasibility of exact computation. In the following sections, we will explore three such estimators and focus on two of their key statistics properties, namely their bias and variance.

\section{Naive Inverse Propensity Scoring}
\label{sec:naive}

%In our discussion of estimators, we will focus on two key properties regarding their statistical efficiency, namely their unbiasedness and their variance. 

%A natural first attempt to estimate $\Utrue{\htar}$ could be to directly report the quantity $\frac{|\Xset||\Yset|}{\sizeall} \sum_{i=1}^\num \sum_{j=1}^\size{i} \Pr(\x{i}{j}) \htar(\y{i}{j} | \x{i}{j}) \del{i}{j}$ over all the data samples in $\Dall$. However, as noted in \cite{swaminathan2015counterfactual}, this estimate is biased since it does not account for the source of the data samples. In particular, predictions favored by the loggers are over-represented which skews the estimate.  

A natural first candidate to explore for the evaluation problem using multiple logging policies as defined above is the well-known inverse propensity score (IPS) estimator. It simply averages over all datapoints, and corrects for the distribution mismatch betweenthe logging policies $\hlog{i}$ and the target policy $\htar$ using a weighting term:

\begin{definition}[Naive IPS Estimator]   
\begin{align*}
\Unaive \equiv \frac{1}{\sizeall} \sum_{i=1}^\num \sum_{j=1}^\size{i}\del{i}{j}\frac{ \htar(\y{i}{j} | \x{i}{j})}{\p{i}{j}}.
\end{align*}
\end{definition}

This is an unbiased estimator as shown below, as long as all logging policies have full support for the new policy $\htar$. 

\begin{definition}[Support] Policy $\h$ is said to have \textit{support} for policy $\h^\prime$ if for all $x \in \Xset$ and $y \in \Yset$, 
\begin{align*}
\delta(x,y)\h^\prime(y|x) \neq 0 \Rightarrow \h(y|x) > 0.
\end{align*}
\end{definition}

\begin{proposition}[Bias of Naive IPS Estimator] 
\label{Lem:unaiveunbiased}
Assume each logging policy $\hlog{i}$ has support for target $\htar$. For $\Dall$ consisting of i.i.d. draws from $\Pr(\Xset)$ and logging policies $\hlog{i}(\Yset|x)$, the naive IPS estimator is unbiased:
\begin{align*}
\mathbb{E}_{\Dall} [\Unaive] 
= \Utrue{\htar}.     
\end{align*}
\end{proposition}

\begin{proof}
By linearity of expectation, 
\begin{align*}
\mathbb{E}_{\Dall}[\Unaive] 
&= \frac{1}{\sizeall} \sum_{i=1}^\num \sum_{j=1}^\size{i} \mathbb{E}_{x \sim \Pr(\Xset), y \sim \hlog{i}(\Yset|x)} \left[\frac{\delta(x, y) \htar(y | x)}{\hlog{i}(y | x)}\right] \\
 &= \frac{1}{\sizeall} \sum_{i=1}^\num \size{i} \sum_{\mathclap{\substack{x \in \Xset \\ y \in \Yset}}} 
\text{Pr}(x) \hlog{i}(y|x) \frac{\delta(x, y) \htar(y | x)}{\hlog{i}(y | x)} \\
 &= \sum_{\mathclap{\substack{x \in \Xset \\ y \in \Yset}}}
\text{Pr}(x) \delta(x, y) \htar(y | x) \\
 &= \mathbb{E}_{x \sim Pr(\Xset)} \mathbb{E}_{y \sim \htar(\Yset|x)} [\delta(x, y)] \\
 &= \Utrue{\htar}.   
\end{align*}
The second equality is valid since each $\hlog{i}$ has support for $\htar$. 
\end{proof}
Note that the requirement that the logging policies $\hlog{i}$ have support for the target policy can be  satisfied by ensuring that $\hlog{i}(y|x) > \epsilon$ when deploying policies.  

We can also characterize the variance of the naive IPS estimator.
\begin{align}
\mathrm{Var}_{\Dall}&[\Unaive] \label{eq:naiveipsvar} \\ 
= & \frac{1}{\sizeall^2} \sum_{i=1}^\num \size{i} \; \bigg( \sum_{\mathclap{\substack{x \in \Xset  \\ y \in \Yset}}}  \frac{(\delta(x, y) \htar(y|x))^2}{\h_i(y | x)} \Pr(x) - \Utrue{\htar}^2 \bigg). \nonumber
\end{align}

Having characterized both the bias and the variance of the Naive IPS Estimator, how does it perform on datasets that come from multiple logging policies?

\subsection{Suboptimality of Naive IPS Estimator}

To illustrate the suboptimality of the Naive IPS Estimator when we have data from multiple logging policies, consider the following toy example where we wish to evaluate a new policy $\htar$ given data from two logging policies $\hlog{1}$ and $\hlog{2}$. For simplicity and without loss of generality, consider logged bandit feedback which consists of one sample from $\hlog{1}$ and another sample from $\hlog{2}$, more specifically, we have two logs $\D{1} = \{(\x{1}{1}, \y{1}{1}, \del{1}{1}, \p{1}{1})\}$, and
    $\D{2} = \{(\x{2}{1}, \y{2}{1}, \del{2}{1}, \p{2}{1})\}$.
There are two possible inputs $x_1, x_2$ and two possible output predictions $y_1, y_2$. The cardinal utility function $\delta$, the input distribution $\Pr(\Xset)$, the target policy $\htar$, and the two logging policies $\hlog{1}$ and $\hlog{2}$ are given in Table~\ref{tab:toy}. 

\begin{table}[t]
\begin{tabular}{lrrr}
    \toprule
          &   $\:\:\:\:\:\:\:\:$    & $\:\:\:\:\:\:\:\:\:\:\:\:x_1$ & $\:\:\:\:\:\:\:\:x_2$ \\
    \midrule
    Pr(x) &       & 0.5   & 0.5 \\
    \midrule
    \multicolumn{1}{l}{\multirow{2}[0]{*}{$\delta(x,y)$}} & $y_1$ & 10    & 1 \\
    \multicolumn{1}{l}{} & $y_2$ & 1     & 10 \\
    \midrule
    \multicolumn{1}{l}{\multirow{2}[0]{*}{$\hlog{1}(y|x)$}} & $y_1$ & 0.2   & 0.8 \\
    \multicolumn{1}{l}{} & $y_2$ & 0.8   & 0.2 \\
    \midrule
    \multicolumn{1}{l}{\multirow{2}[0]{*}{$\hlog{2}(y|x)$}} & $y_1$ & 0.9   & 0.1 \\
    \multicolumn{1}{l}{} & $y_2$ & 0.1   & 0.9 \\
    \midrule
    \multicolumn{1}{l}{\multirow{2}[0]{*}{$\htar(y|x)$}} & $y_1$ & 0.8   & 0.2 \\
    \multicolumn{1}{l}{} & $y_2$ & 0.2   & 0.8 \\
    \bottomrule
    \end{tabular}%    
    \caption{Dropping data samples from logging policy $\hlog{1}$ lowers the variance of the naive and balanced IPS estimators when estimating the utility of $\htar$. }
    \label{tab:toy}
\end{table}

From the table, we can see that the target policy $\htar$ is similar to logging policy $\hlog{2}$, but that it is substantially different from $\hlog{1}$. Since the mismatch between target and logging policy enters the IPS estimator as a ratio, one would like to keep that ratio small for low variance. That, intuitively speaking, means that samples from $\hlog{2}$ result in lower variance than samples from $\hlog{1}$, and that the $\hlog{1}$ samples may be adding a large amount of variability to the estimate. Indeed, it turns out that simply omitting the data from $\D{1}$ greatly improves the variance of the estimator.  Plugging the appropriate values into the variance formula in Equation~(\ref{eq:naiveipsvar}) shows that the variance $\var{\Unaive}$ is reduced from $64.27$ to $4.27$ by dropping the sample from the first logging policy $\hlog{1}$. Intuitively, the variance of $\Unaive$ suffers because higher variance samples from one logging policy drown out the signal from the lower variance samples to an extent that can even dominate the benefit of having more samples. Thus, $\Unaive$ fails to make the most of the available log data by combining it in an overly naive way. 

Under closer inspection of Equation~(\ref{eq:naiveipsvar}), the fact that deleting data helps improve variance also makes intuitive sense. Since the overall variance contains the sum of variances over all individual samples, one can hope to improve variance by leaving out high-variance samples. This motivates the estimators we introduce in the following sections, and we will show how weighting samples generalizes this variance-minimization strategy. 
%\ts{I feel this paragraph is a bit shaky -- it also depends on the number of samples, and mentioning that it contains a sum doesn't account for that.}

\section{Estimator from Multiple Importance Sampling}

Having seen that $\Unaive$ has suboptimal variance, we first explore an alternative estimator used in multiple importance sampling \cite{mcbook}. We begin with a brief review of multiple importance sampling.

Suppose there is a target distribution $p$ on $\mathcal{S} \subseteq \mathbb{R}^d$, a function $f$, and $\mu = \mathbb{E}_p(f(\mathbf{X})) = \int_\mathcal{S} f(x) p(x) \mathrm{d} x$ is the quantity to be estimated. The function $f$ is observed only at the sampled points. In multiple importance sampling, $n_j$ observations $x_{ij} \sim \mathbf{X}$, $i \in [n_j]$ are taken from sampling distributions $q_j$ for $j = 1, \ldots, J$. An unbiased estimator that is known to have low variance in this case is the \textit{balance heuristic} estimate \cite{mcbook};

\begin{align*}
   \tilde{\mu}_{\alpha} = \frac{1}{n} \sum_{j=1}^{J} \sum_{i=1}^{n_j} \frac{f(x_{ij}) p(x_{ij})}{\sum_{j=1}^J \alpha_j q_j(x_{ij})},
\end{align*}

where $n = \sum_{j=1}^{J} n_j$, and $\alpha_j = \frac{n_j}{n}$. Directly mapping the above to our setting, we define the Balanced IPS Estimator as follows.

\begin{definition}[Balanced IPS Estimator]    
\begin{align*}
  \Ubal = \frac{1}{\sizeall} \sum_{i=1}^\num \sum_{j=1}^\size{i} \del{i}{j} \frac{\htar(\y{i}{j} | \x{i}{j})}{\havg(\y{i}{j} | \x{i}{j})}, 
\end{align*}
where for all $x \in \Xset$ and $y \in \Yset$, $\havg(y | x) =  \frac {\sum_{i=1}^\num \size{i} \hlog{i}(y |x)}{\sizeall}$.
\end{definition}

Note that $\havg$ is a valid policy since the convex combination of probability distributions is a probability distribution. The balanced IPS estimator $\Ubal$ is also unbiased. Note that it now suffices that $\havg$ has support, but not necessarily that each individual $\hlog{i}$ has support.

\begin{proposition}[Bias of Balanced IPS Estimator]
Assume the policy $\havg$ has support for target $\htar$. For $\Dall$ consisting of i.i.d. draws from $\Pr(\Xset)$ and logging policies $\hlog{i}(\Yset|x)$, the Balanced IPS Estimator is unbiased:
\begin{align*}
\mathbb{E}_{\Dall} [\Ubal] 
= \Utrue{\htar}.     
\end{align*}
\end{proposition}

\begin{proof}
By linearity of expectation, 
\begin{align*}
  \mathbb{E}_{\Dall}[\Ubal] 
      &= \frac{1}{\sizeall} \sum_{i=1}^\num \sum_{j=1}^\size{i} \mathbb{E}_{x \sim \Pr(\Xset), y \sim \hlog{i}(\Yset|x)} \left[\frac{\delta(x, y) \htar(y | x)}{\havg(y | x)}\right] \\
      &= \frac{1}{\sizeall} \sum_{i=1}^\num \size{i} \sum_{\mathclap{\substack{x \in \Xset \\ y \in \Yset}}} \text{Pr}(x) \hlog{i}(y|x) \frac{\delta(x, y) \htar(y | x)}{\havg(y|x)} \\
&= \frac{1}{\sizeall} \sum_{x \in \mathcal{X}, y \in \mathcal{Y}} \frac{\text{Pr}(x)\delta(x, y) \htar(y | x)}{\havg(y|x)}\sum_{i=1}^\num \size{i} \hlog{i}(y|x) \\
&= \frac{1}{\sizeall} \sum_{x \in \mathcal{X}, y \in \mathcal{Y}} \frac{\text{Pr}(x)\delta(x, y) \htar(y | x)}{\frac {\sum_{i=1}^\num \size{i} \hlog{i}(y |x)}{\sizeall}}\sum_{i=1}^\num \size{i} \hlog{i}(y|x) \\
&= \sum_{x \in \mathcal{X}, y \in \mathcal{Y}} \text{Pr}(x)\delta(x, y) \htar(y | x) \\
 &= \mathbb{E}_{x \sim Pr(\Xset)} \mathbb{E}_{y \sim \htar(\Yset|x)} [\delta(x, y)] \\
 &= \Utrue{\htar}.     
\end{align*}
The second equality is valid since $\havg$ has support for $\htar$. 
\end{proof}

The variance of $\Ubal$ can be computed as follows:
\begin{align*}
\mathrm{Var}_{\Dall} [\Ubal] &= \frac{1}{\sizeall^2} \sum_{i=1}^\num \size{i} \Bigg ( \sum_{\mathclap{\substack{x \in \Xset  \\ y \in \Yset}}} \frac{(\delta(x, y) \htar(y|x))^2}{\havg(y | x)^2} \hlog{i}(y|x) \Pr(x) \\
&- \bigg(\sum_{\mathclap{\substack{x \in \Xset  \\ y \in \Yset}}} \frac{(\delta(x, y) \htar(y|x))}{\havg(y | x)} \hlog{i}(y|x) \Pr(x) \bigg)^2 \Bigg).
\end{align*}

A direct consequence of Theorem 1 in \cite{veach1995optimally} is that the variance of the balanced estimator is bounded above by the variance of the naive estimator plus some positive term that depends on $\Utrue{\htar}$ and the log sizes $\size{i}$. 

%\begin{align*}
%    \var{\Ubal} \leq \var{\Unaive} + \big (\frac{1}{\min_{i} \size{i}} - %\frac{1}{\sizeall} \big) \Utrue{\htar}^2
%\end{align*} 

Here, we provide a stronger result that does not require an extra positive term for the inequality to hold.

\begin{theorem}
Assume each logging policy $\hlog{i}$ has support for target $\htar$. We then have that
\begin{align*}
    \var{\Ubal} \leq \var{\Unaive}.
\end{align*}
\end{theorem}

\begin{proof} 
From Equation~\ref{eq:naiveipsvar}, we have the following expression.

\begin{align*}
\mathrm{Var}_{\Dall} [\Unaive]&  \\ 
= & \frac{1}{\sizeall^2} \sum_{i=1}^\num \size{i} \bigg( \sum_{\mathclap{\substack{x \in \Xset  \\ y \in \Yset}}} \frac{(\delta(x, y) \htar(y|x))^2}{\h_i(y | x)} \Pr(x) - \Utrue{\htar}^2 \bigg). \nonumber
\end{align*}

%\tj{fixed typo $\h_i$ in formula above.}

For convenience, and without loss of generality, assume $\size{i} = 1$ $\forall i$, and therefore, $\sizeall = \num$. This is easily achieved by re-labeling the logging policies so that each data-sample comes from a distinctly labeled policy (note that we don't need the logging policies to be distinct in our setup).  Also, for simplicity, let $c(x,y) = \delta(x,y) \htar(y|x)$. Then 

\begin{align*}
&\mathrm{Var}_{\Dall} [\Unaive] \geq \mathrm{Var}_{\Dall} [\Ubal] \\
&\Leftrightarrow \sum_{\mathclap{\substack{x \in \Xset  \\ y \in \Yset}}} c^2(x,y) \Pr(x) 
\big ( \sum_{i=1}^{\num} \frac{1}{\hlog{i}(y|x)} \big ) - \num \Utrue{\htar}^2 \\
&\geq \sum_{\mathclap{\substack{x \in \Xset  \\ y \in \Yset}}} \frac{c^2(x,y) \Pr(x)}{\havg(y|x)^2} \bigg( \sum_{i=1}^{\num} \hlog{i}(y|x) \bigg) - \sum_{i=1}^{\num} \bigg(\sum_{\mathclap{\substack{x \in \Xset  \\ y \in \Yset}}} \frac{c(x, y) \Pr(x)}{\havg(y | x)} \hlog{i}(y|x) \bigg)^2   
\end{align*} 

Thus, it is sufficient to show the following two inequalities
\begin{align} \label{eq:one}
\sum_{i=1}^{\num} \bigg(\sum_{\mathclap{\substack{x \in \Xset  \\ y \in \Yset}}} \frac{c(x, y) \Pr(x)}{\havg(y | x)} \hlog{i}(y|x) \bigg)^2 \geq \num \Utrue{\htar}^2  
\end{align}
and for all relevant $x, y$
\begin{align} \label{eq:two}
\sum_{i=1}^{\num} \frac{1}{\hlog{i}(y|x)} \geq \frac{1}{\havg(y|x)^2} \bigg( \sum_{i=1}^{\num} \hlog{i}(y|x) \bigg)
\end{align}

We get Equation~\ref{eq:one} by applying Cauchy-Schwarz as follows
\begin{align*}
&\bigg( \sum_{i=1}^{\num} 1^2 \bigg) \bigg( \sum_{i=1}^{\num} \bigg(\sum_{\mathclap{\substack{x \in \Xset  \\ y \in \Yset}}} \frac{c(x, y) \Pr(x)}{\havg(y | x)} \hlog{i}(y|x) \bigg)^2 \bigg) \\
&\geq \bigg( \sum_{i=1}^{\num} \sum_{\mathclap{\substack{x \in \Xset  \\ y \in \Yset}}} \frac{c(x, y) \Pr(x)}{\havg(y | x)} \hlog{i}(y|x) \bigg)^2 \\ 
&\Rightarrow \sum_{i=1}^{\num} \bigg(\sum_{\mathclap{\substack{x \in \Xset  \\ y \in \Yset}}} \frac{c(x, y) \Pr(x)}{\havg(y | x)} \hlog{i}(y|x) \bigg)^2 \\
&\geq \bigg( \sum_{\mathclap{\substack{x \in \Xset  \\ y \in \Yset}}} \frac{c(x, y) \Pr(x)}{\frac{1}{\num} \sum_{i=1}^{\num} \hlog{i}(y|x)} \sum_{i=1}^{\num} \hlog{i}(y|x) \bigg)^2 = \num \Utrue{\htar}^2   
\end{align*}

%\tj{removed $\frac{1}{n}$ in formula above.}

Another application of Cauchy-Schwarz gives us Equation~\ref{eq:two} in the following way
\begin{align*}
&\bigg(\sum_{i=1}^{\num} \frac{1}{\hlog{i}(y|x)} \bigg) \bigg( \sum_{i=1}^{\num} \hlog{i}(y|x) \bigg) \geq \num^2 \\
&\Rightarrow  \sum_{i=1}^{\num} \frac{1}{\hlog{i}(y|x)} \geq \frac{1}{(\frac{1}{\num} \sum_{i=1}^{\num} \hlog{i}(y|x))^2} \sum_{i=1}^{\num} \hlog{i}(y|x) \\
&= \frac{1}{\havg(y|x)^2} \bigg( \sum_{i=1}^{\num} \hlog{i}(y|x) \bigg)
\end{align*}

\end{proof}

%In fact, the analysis in Section~\ref{sec:ubal} provides strong empirical %evidence that $\Ubal$ has generally lower variance than $\Unaive$.   

Returning to our toy example in Table~\ref{tab:toy}, we can check the variance reduction provided by $\Ubal$ over $\Unaive$. The variance of the Balanced IPS Estimator is $\var{\Ubal} \approx 12.43$, which is substantially smaller than $\var{\Unaive} \approx 64.27$ for the naive estimator using all the data $\Dall = \D{1} \bigcup \D{2}$. However, the Balanced IPS Estimator still improves when removing $\D{1}$. In particular, notice that when using only $\D{2}$, the variance of the Balanced IPS Estimator is $\var{\Ubal} = \var{\Unaive} \approx 4.27 < 12.43$.  Therefore, even the variance of $\Ubal$ can be improved in some cases by dropping data.

\section{Weighted IPS Estimator}

We have seen that the variances of both the Naive and the Balanced IPS estimators can be reduced by removing some of the data points. More generally, we now explore estimators that re-weight samples from various logging policies based on their relationship with the target policy. This is similar to ideas that are used in Adaptive Multiple Importance Sampling~\cite{cornuet2012adaptive,elvira2015generalized} where samples are also re-weighted in each sampling round. In contrast to the latter scenario, here we assume the logging policies to be fixed, and we derive closed-form formulas for variance-optimal estimators. The general idea of the weighted estimators that follow is to compute a weight for each logging policy that captures the mismatch between this policy and the target policy. In order to characterize the relationship between a logging policy and the new policy to be evaluated, we define the following \emph{divergence}. This formalizes the notion of mismatch between the two policies in terms of the Naive IPS Estimator variance. 

\begin{definition}[Divergence] Suppose policy $\h$ has support for target policy $\htar$. Then the divergence from $\h$ to $\htar$ is 
\begin{align*}
\diverg{\htar}{\h} &\equiv \mathrm{Var}_{x \sim \Pr(\Xset), y \sim \h(\Yset|x)} \left[\frac{\delta(x,y) \htar(y|x)}{\h(y|x)}\right] \\
%& = \sum_{\mathclap{\substack{x \in \Xset \\ y \in \Yset}}} \big(\frac{\delta(x, y) \htar(y|x)}{\h(y | x)}\big)^2 \Pr(x) \h(y | x) \\
%& - \Big(\sum_{\mathclap{\substack{x \in \Xset \\ y \in \Yset}}}\frac{ \delta(x, y) \htar(y|x) \text{Pr}(x) \h(y|x)}{\h(y|x)}\Big)^2\\
& = \sum_{\mathclap{\substack{x \in \Xset  \\ y \in \Yset}}}  \frac{(\delta(x, y) \htar(y|x))^2}{\h(y | x)} \Pr(x) - \Utrue{\htar}^2 .
\end{align*}
Recall that $\Utrue{\htar}$ is the utility of policy $\htar$.
\label{def:div}
\end{definition}

Note that $\diverg{\htar}{\h}$ is not necessarily minimal when $\h = \htar$. In fact, it can easily be seen by direct substitution that $\diverg{\htar}{\htar_{imp}} = 0$ where $\htar_{imp}$ is the optimal importance sampling distribution for $\htar$ with $\htar_{imp} (y|x) \propto \delta(x,y) \htar(y|x)$. Nevertheless, informally, the divergence from a logging policy to the target policy is small when the logging policy assigns similar propensities to $(x,y)$ pairs as the importance sampling distribution for the target policy. Conversely, if the logging policy deviates significantly from the importance sampling distribution, then the divergence is large. 
Based on this notion of divergence, we propose the following weighted estimator:

\begin{definition}[Weighted IPS Estimator] Assume $\diverg{\htar}{\hlog{i}} > 0$ for all $1 \leq i \leq \num$. 
\begin{align*}
  \Uwt = \sum_{i=1}^\num \wtopt{i} \sum_{j=1}^\size{i} \frac{\del{i}{j} \htar(\y{i}{j} | \x{i}{j})}{\p{i}{j}}
\end{align*}
where the weights $\wtopt{i}$ are set to
\begin{align}
\label{eq:weighted_optimal_weights}
    \wtopt{i} &= \frac{1}{\diverg{\htar}{\hlog{i}}\sum_{j=1}^\num \frac{\size{j}}{\diverg{\htar}{\hlog{j}}}}. 
\end{align}

\end{definition}

Note that the assumption $\diverg{\htar}{\hlog{i}} > 0$ is easily satisfied as long as the logging policy is not exactly equal to the optimal importance sampling distribution of the target policy $\htar$. 
This is very unlikely given that the utility of the new policy is unknown to us in the first place.
%If this were so, then a single sample from such a logging policy would yield a zero-variance unbiased IPS estimate of $\Utrue{\htar}$. \aaa{Discuss?}

We will show that the Weighted IPS Estimator is optimal in the sense that any other convex combination by $\wt{i}$ that ensures unbiasedness does not give a smaller variance estimator. First, we have a simple condition for unbiasedness:

\begin{proposition}[Bias of Weighted IPS Estimator] \label{prop:biasweighted}
Assume each logging policy $\hlog{i}$ has support for target policy $\htar$. Consider the estimator
\begin{align*}
    \Ugenweight = \sum_{i=1}^\num \wt{i} \sum_{j=1}^\size{i} \frac{\del{i}{j} \htar(\y{i}{j} | \x{i}{j})}{\p{i}{j}}
\end{align*}
such that $\wt{i} \geq 0$ and $\sum_{i=1}^\num \wt{i} \size{i} = 1$. For $\Dall$ consisting of i.i.d. draws from $\Pr(\Xset)$ and logging policies $\hlog{i}(\Yset|x)$, the above estimator is unbiased:
\begin{align*}
\mathbb{E}_{\Dall} [\Ugenweight] 
= \Utrue{\htar}.     
\end{align*}
In particular, $\Uwt$ is unbiased.
\end{proposition}

\begin{proof} Following the proof of Proposition~\ref{Lem:unaiveunbiased},
\begin{align*}
\mathbb{E}_{\Dall}[\Ugenweight] 
&=  \sum_{i=1}^\num \wt{i} \sum_{j=1}^\size{i} \mathbb{E}_{x \sim \Pr(\Xset), y \sim \hlog{i}(\Yset|x)} \bigg[\frac{\delta(x, y) \htar(y | x)}{\hlog{i}(y | x)}\bigg] \\
 &= \Utrue{\htar} \sum_{i=1}^\num \wt{i} \size{i}  = \Utrue{\htar}.      
\end{align*}
Moreover, $\sum_{i=1}^\num \wtopt{i} \size{i} = 1$, which implies $\Uwt$ is unbiased.
\end{proof}

Notice that making the weights equal reduces $\Ugenweight$ to $\Unaive$. Furthermore, dropping samples from logging policy $\hlog{i}$ is equivalent to setting $\wt{i} = 0$. 

To prove variance optimality, note that the variance of the Weighted IPS Estimator for a given set of weights $\wt{1}, ..., \wt{\num}$ can be written in terms of the divergences.
\begin{align}
\label{eq:weighted_var}
\mathrm{Var}_{\Dall} [\Ugenweight] 
= \sum_{i=1}^\num \wt{i}^2 \size{i} \diverg{\htar}{\hlog{i}}.
\end{align}

We now prove the following theorem:
\begin{theorem} 
Assume each logging policy $\hlog{i}$ has support for target policy $\htar$, and $\diverg{\htar}{\hlog{i}} > 0$. Then, for any estimator of the form $\Ugenweight$ as defined in Proposition \ref{prop:biasweighted}
\begin{align*}
    \var{\Uwt} =  \frac{1}{\sum_{i=1}^\num \frac{\size{i}}{\diverg{\htar}{\hlog{i}}}} \leq \var{\Ugenweight}.
\end{align*}
\end{theorem}
\begin{proof} The expression for the variance of $\Uwt$ can be verified to be as stated by directly substituting $\wtopt{i}$~\eqref{eq:weighted_optimal_weights} into the variance expression in Equation \eqref{eq:weighted_var}. Next, by the Cauchy-Schwarz inequality, 
\begin{align*}
  &\left(\sum_{i=1}^\num \wt{i}^2 \size{i} \diverg{\htar}{\hlog{i}} \right) \left(\sum_{i=1}^\num \frac{\size{i}}{\diverg{\htar}{\hlog{i}}}\right) \geq \left(\sum_{i=1}^\num \wt{i} \size{i}\right)^2 = 1  \\
  &\Rightarrow \var{\Ugenweight} \geq \var{\Uwt}
\end{align*}
\end{proof}

Returning to the toy example in Table~\ref{tab:toy}, the divergence values are $\diverg{\htar}{\hlog{1}} \approx 252.81$ and $\diverg{\htar}{\hlog{2}} \approx 4.27$. This leads to weights $\wtopt{1} \approx 0.02$ and $\wtopt{2} \approx 0.98$, resulting in $\var{\Uwt} \approx 4.19 < 4.27$ on $\Dall = \D{1} \bigcup \D{2}$. Thus, the weighted IPS estimator does better than the naive IPS estimator (including the case when $\D{1}$ is dropped) by optimally weighting all the available data. 

Note that computing the optimal weights $\wt{i}$ exactly requires access to the utility function $\delta$ everywhere in order to compute the divergences $\diverg{\htar}{\hlog{i}}$. However, in practice, $\delta$ is only known at the collected data samples, and the weights must be estimated. In Section~\ref{sec:weightest} we discuss a simple strategy for doing so, along with an empirical analysis of the procedure.

\subsection{Quantifying the Variance Reduction}
\label{sec:varred}

The extent of variance reduction provided by the Weighted IPS Estimator over the Naive IPS Estimator depends only on the relative proportions of divergences and the log data sizes of each logging policy. The following proposition quantifies the variance reduction.

\begin{proposition}
Let $v_i = \frac{\diverg{\htar}{\hlog{i}}}{\diverg{\htar}{\hlog{\num}}}$ be the ratio of divergences and $r_i = \frac{\size{i}}{\size{\num}}$ be the ratio of sample sizes of policy $i$ and policy $\num$. Then the reduction denoted as $\gamma$ is
\begin{displaymath}
  \gamma \equiv \frac{\var{\Uwt}}{\var{\Unaive}} = \frac{(\sum_{i=1}^\num r_i)^2}{(\sum_{i=1}^\num r_i v_i) (\sum_{i=1}^\num \frac{r_i}{v_i})} \leq 1.
\end{displaymath}
\end{proposition}

\begin{proof}
Substituting the expressions for the two variances, we get that
\begin{displaymath}
  \frac{\var{\Uwt}}{\var{\Unaive}} = \frac{(\sum_{i=1}^\num \size{i})^2}{(\sum_{i=1}^\num \size{i} \diverg{\htar}{\hlog{i}}) (\sum_{i=1}^\num \frac{\size{i}}{\diverg{\htar}{\hlog{i}}})} 
\end{displaymath}

So, normalizing by $\diverg{\htar}{\hlog{n}}$ and $\size{n}$, gives the desired expression. Applying the Cauchy-Schwarz inequality gives the upper bound. 
\end{proof}

For the case of just two logging policies, $n = 2$, it is particularly easy to compute the maximum improvement in variance of the Weighted IPS Estimator over the Naive estimator. The reduction $\gamma$ is $\gamma = \frac{(r_1 + 1)^2 v_1}{(r_1 v_1 + 1)(r_1 + v_1)}$,  which ranges between $0$ and $1$ depending on $r_1$ and $v_1$. The benefit of the weighted estimator over the naive estimator is greatest when the logging policies differ substantially, and there are equal amounts of log data from the two logging policies. Intuitively, this is because the weighted estimator mitigates the defect in the naive estimator due to which abundant high variance samples drown out the signal from the equally abundant low variance samples. On the other hand, the scope for improvement is less when the logging policies are similar or when there are disproportionately many samples from one logging policy.

\section{Empirical Analysis}

In this section, we empirically examine the properties of the proposed estimators. To do this, we create a controlled setup in which we have logging policies of different utilities, and try to estimate the utility of a fixed new policy.
We illustrate key properties of our estimators in the concrete setting of CRF policies for multi-label classification, although the estimators themselves are applicable to arbitrary stochastic policies and structured output spaces.

\subsection{Setup}
\label{sec:setup}

\begin{table}[t]
    \begin{tabular}{rrrrr   }
        \toprule
        Name   & \# features &   \# labels & $n_{train}$ & $n_{test}$ \\
        \midrule
        Scene   & 294 & 6 & 1211 & 1196 \\
        Yeast   & 103  & 14 & 1500 & 917 \\
        LYRL    & 47236 & 4 & 23149 & 781265 \\
        \bottomrule
    \end{tabular}
    \caption{Corpus statistics for different multi-label datasets from the LibSVM repository. LYRL was post-processed so that only top level categories were treated as labels}
    \label{tab:datasets}
\end{table}    

\begin{figure*}[t]
    \centering
    \includegraphics[width=0.33\textwidth]{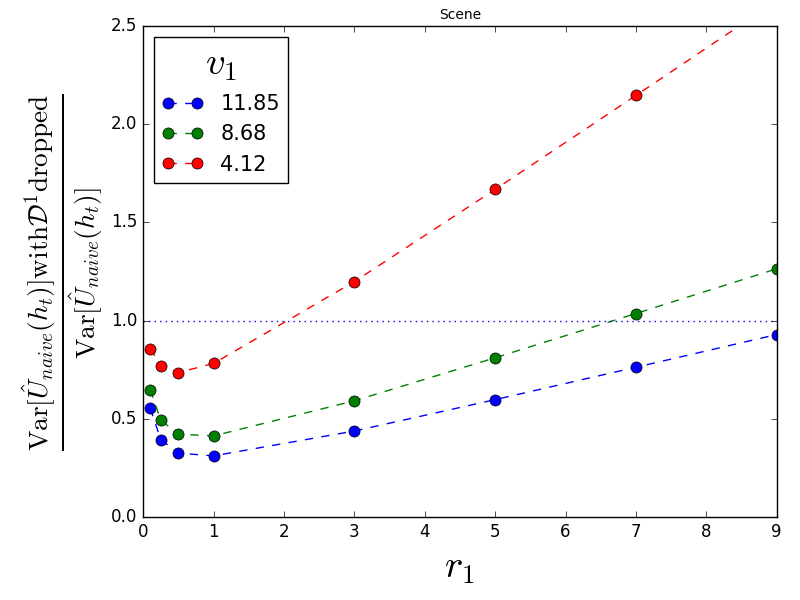} \hfill \includegraphics[width=0.33\textwidth]{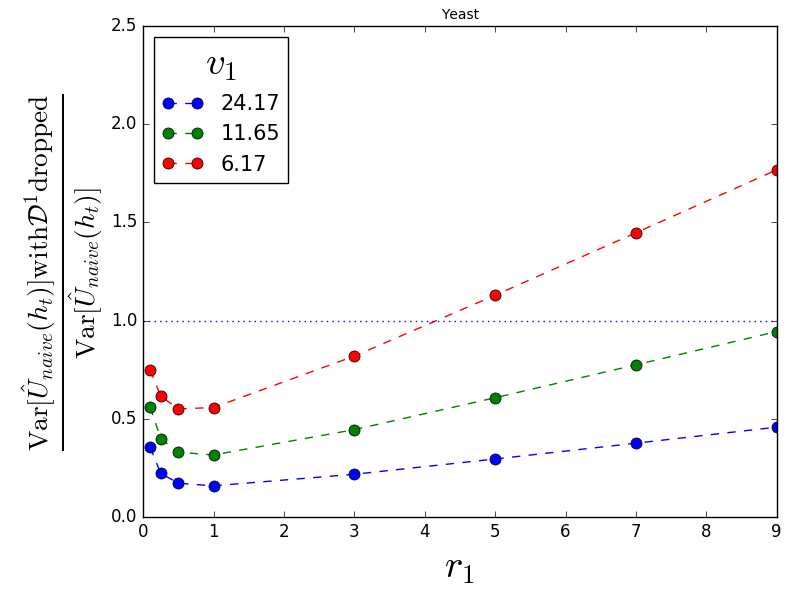} \hfill \includegraphics[width=0.33\textwidth]{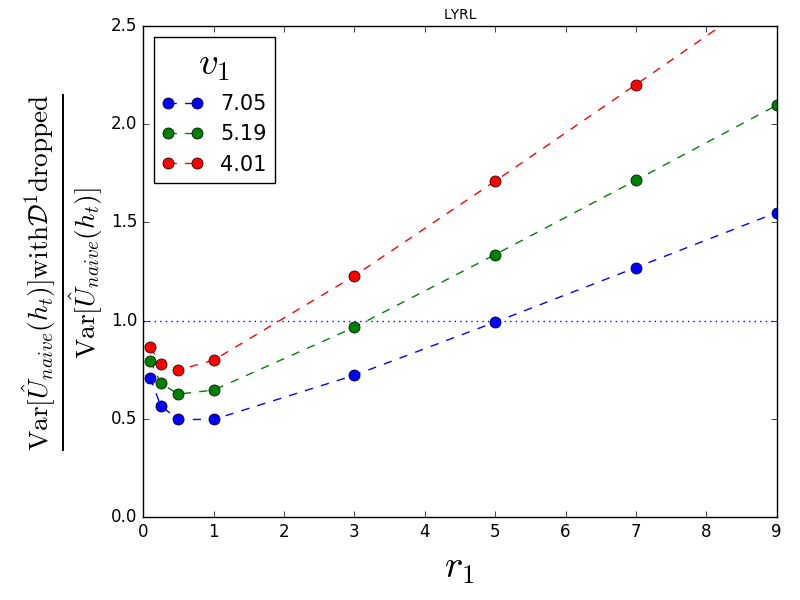}
    \caption{Variance of the Naive IPS Estimator using only $\hlog{2}$ relative to the variance of the Naive IPS Estimator using data from both $\hlog{1}$ and $\hlog{2}$ for different $\hlog{1}$ as the relative sample size changes. Dropping data can lower the variance of Naive IPS Estimator in many cases.}
    \label{fig:drop}
\end{figure*}

\begin{figure*}[t]
    \centering
    \includegraphics[width=0.33\textwidth]{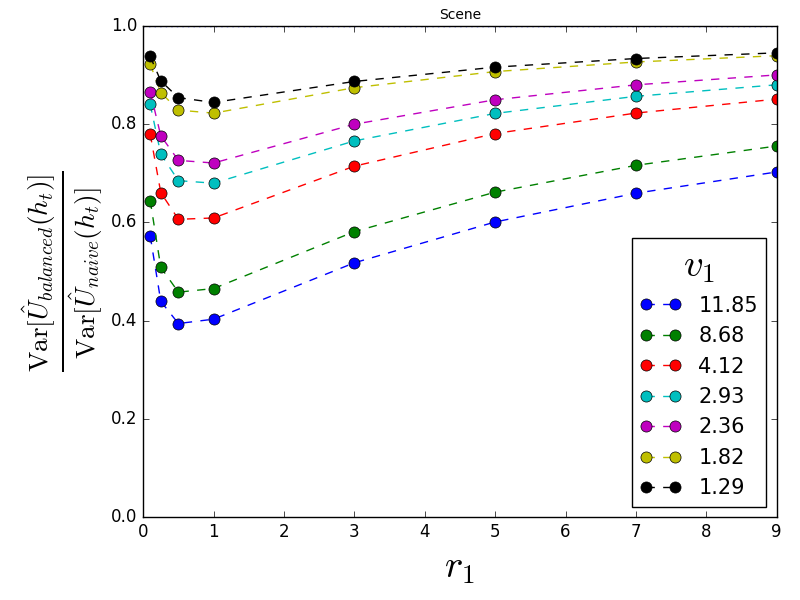} \hfill \includegraphics[width=0.33\textwidth]{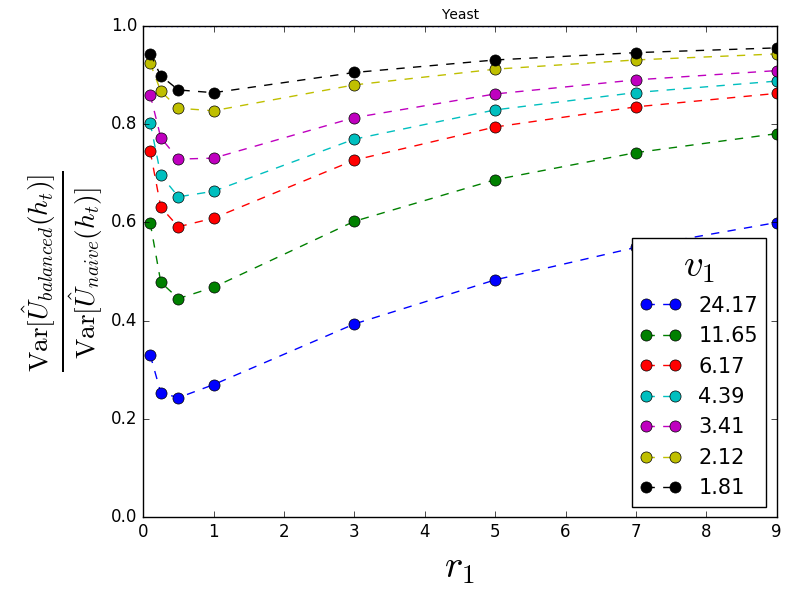} \hfill \includegraphics[width=0.33\textwidth]{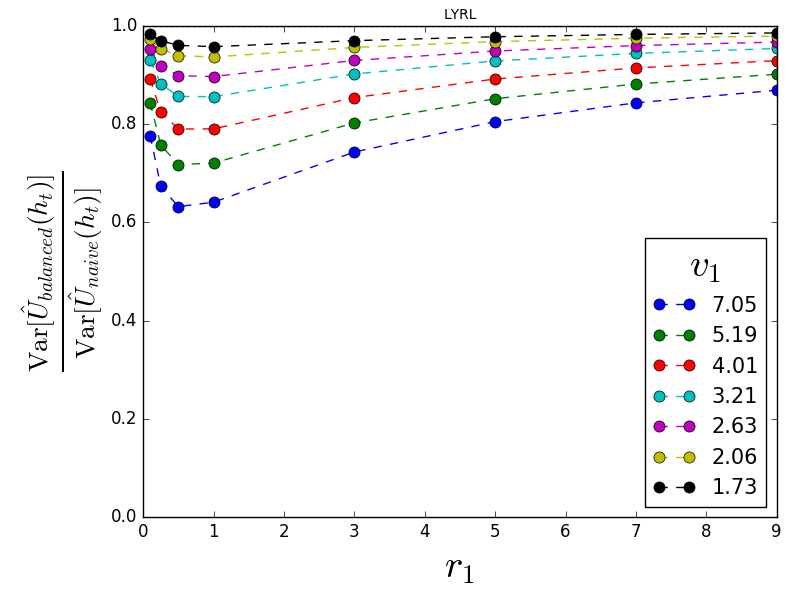}
    \caption{Variance of the Balanced IPS Estimator relative to the variance of the Naive IPS Estimator for different $\hlog{1}$ as the relative sample size changes. The Balanced IPS Estimator can have substantially smaller variance than the Naive IPS Estimator.}
    \label{fig:r2}
\end{figure*}

We choose multi-label classification for our experiments because of the availability of a rich feature space $\Xset$ and an easily scalable label space $\Yset$. Three multi-label datasets from the LibSVM repository with varying feature dimensionalities, number of class labels, and number of training samples available are used. The corpus statistics are as summarized in Table~\ref{tab:datasets}. 

Since these datasets involve multi-label classification, the output space is $\mathcal{Y} = \{0,1\}^q$, i.e., the set of all possible labels one can generate given a set of $q$ labels.
%involves an input $x \in \mathbb{R}^p$ and prediction $y \in \{0,1\}^q$. Using the notation in Section~\ref{sec:problem}, the set of inputs $\mathcal{X}$ was taken to be the inputs in the test set, and $\mathcal{Y} = \{0,1\}^q$ as the set of all possible predictions or outputs. 
The input distribution $\Pr(\Xset)$ is the empirical distribution of inputs as represented in the test set. The utility function $\delta(x,y)$ is simply the number of correctly assigned labels in $y$ with respect to the given ground truth label $y^*$. 

To obtain policies with different utilities in a systematic manner, we train conditional random fields (CRFs) on incrementally varying fractions of the labeled training set. CRFs are convenient since they provide explicit probability distributions over possible predictions conditioned on an input. However, nothing in the following analysis is specific to using CRFs as the stochastic logging policies, and note that the target policy need not be stochastic at all.

For simplicity and ease of interpretability, we use two logging policies in the following experiments. To generate these logging policies, we vary the training fraction for the first logging policy $\hlog{1}$ over $0.02, 0.05, 0.08, 0.11, 0.14, 0.17, 0.20$, keeping the training fractions for the second logging policy $\hlog{2}$ fixed at $0.30$. Similarly, we generate a CRF classifier representing the target policy $\htar$ by training on $0.35$ fraction of the data. The effect is that we now get three policies where the second logging policy is similar to the target while the similarity of the first logging policy varies over a wide range. This results in a wide range of relative divergences 
\begin{equation*}
    v_1 = \frac{\diverg{\htar}{\hlog{1}}}{\diverg{\htar}{\hlog{2}}}
\end{equation*}
for the first logging policy on which the relative performance of the estimators depends.  %\tj{How about calling this just $v$ instead of $v_1$, since we never have a $v_2$.}

We compare pairs of estimators based on their relative variance since all the estimators being considered are unbiased (so, relative variance 1 signifies the estimators being compared have the same variance). Since the variance of the different estimators scales inversely proportional to the total number of samples, the ratio of their variances depends only on the relative size of the two data logs 
\begin{equation*}
    r_1 = \frac{\size{1}}{\size{2}},
\end{equation*}
but not on their absolute size. We therefore report results in terms of relative size where we vary $r_1 \in \{0.1, 0.25, 0.5, 1, 3, 5, 7, 9\}$ to explore a large range of data imbalances. %\tj{How about calling this just $r$ instead of $r_1$, since we never have a $r_2$.}

For a fixed set of CRFs as logging and target policies, and the relative size of the data logs, the ratio of the variances of the different estimators can be computed exactly since the CRFs provide explicit distributions over $\Yset$, and $\Xset$ is based on the test set. We therefore report exact variances in the following. In addition to the exactly computed variances, we also did some bandit feedback simulations to verify the experiment setup. We employed the Supervised $\mapsto$ Bandit conversion method~\cite{Agarwal14tamingthe}. In this method, we iterate over the test features $x$, sample some prediction $y$ from the logging policy $\hlog{i}(\Yset|x)$ and record the corresponding loss and propensity to generate the logged data-sets $\D{i}$. For various settings of logging policies and amounts of data, we sampled bandit data and obtained estimator values over hundreds of iterations. We then computed the empirical mean and variance of the different estimates to make sure that the estimators were indeed unbiased and closely matched the theoretical variances reported above.

\begin{figure*}[t]
    \centering
    \includegraphics[width=0.33\textwidth]{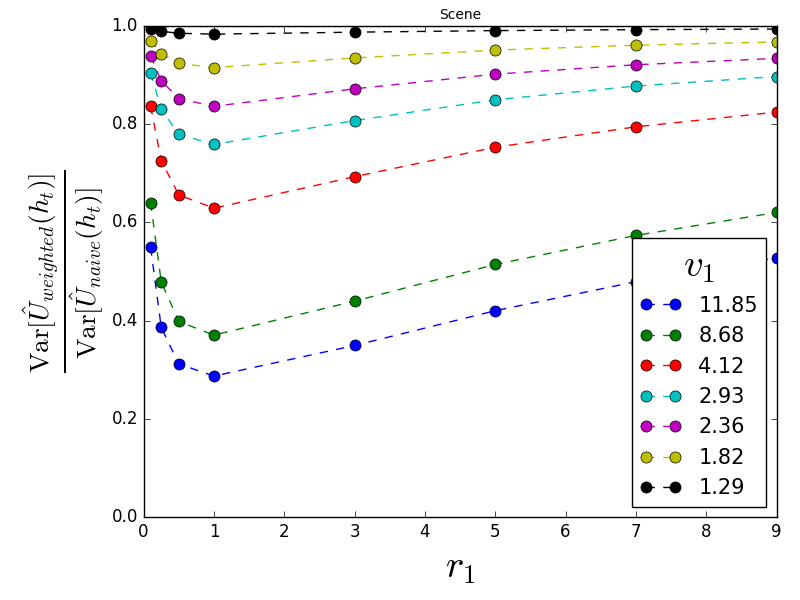} \hfill \includegraphics[width=0.33\textwidth]{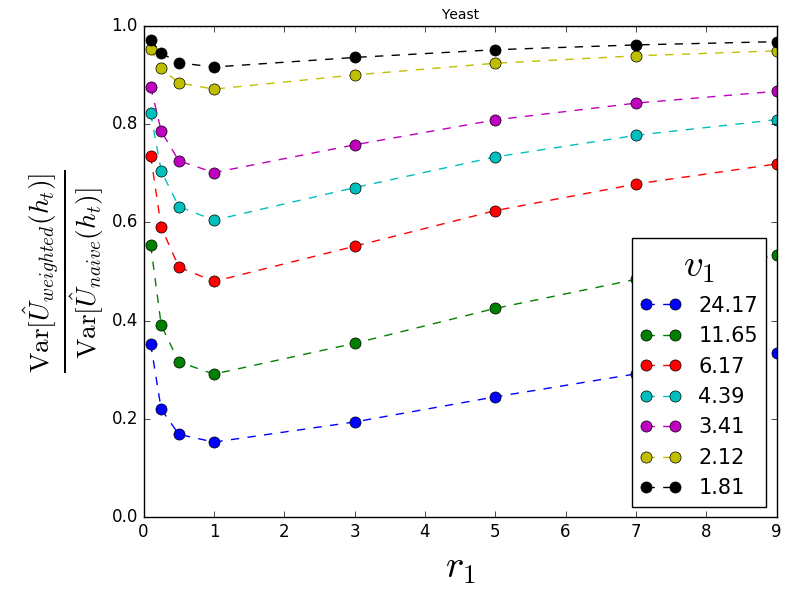} \hfill \includegraphics[width=0.33\textwidth]{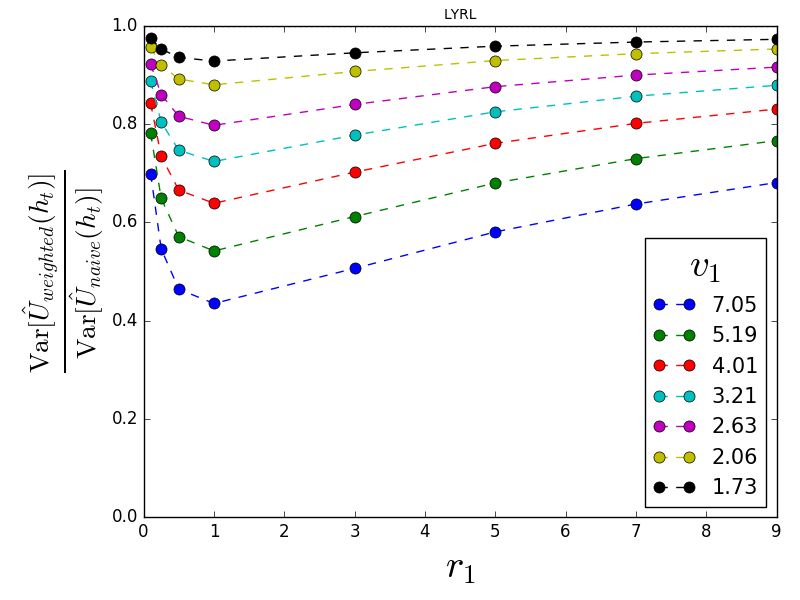}
    \caption{Variance of the Weighted IPS Estimator relative to the variance of the Naive IPS Estimator for different $\hlog{1}$ as the relative sample size changes. The Weighted IPS Estimator can have substantially smaller variance than the Naive IPS Estimator.}
    \label{fig:optr3}
\end{figure*}

\begin{figure*}[t]
    \centering
    \includegraphics[width=0.33\textwidth]{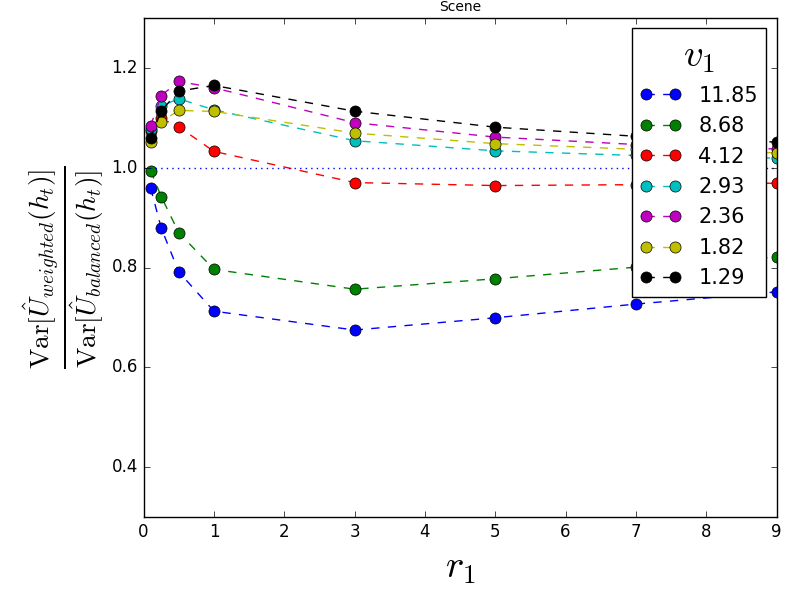} \hfill \includegraphics[width=0.33\textwidth]{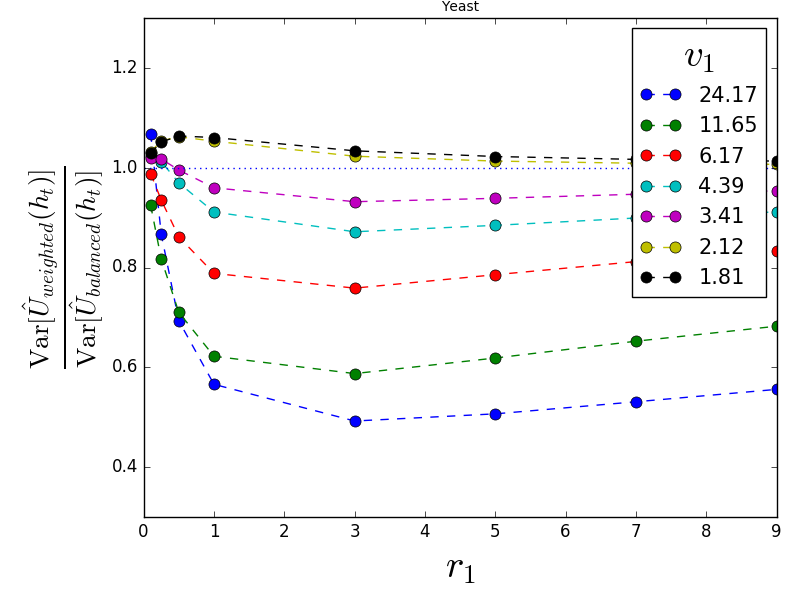} \hfill \includegraphics[width=0.33\textwidth]{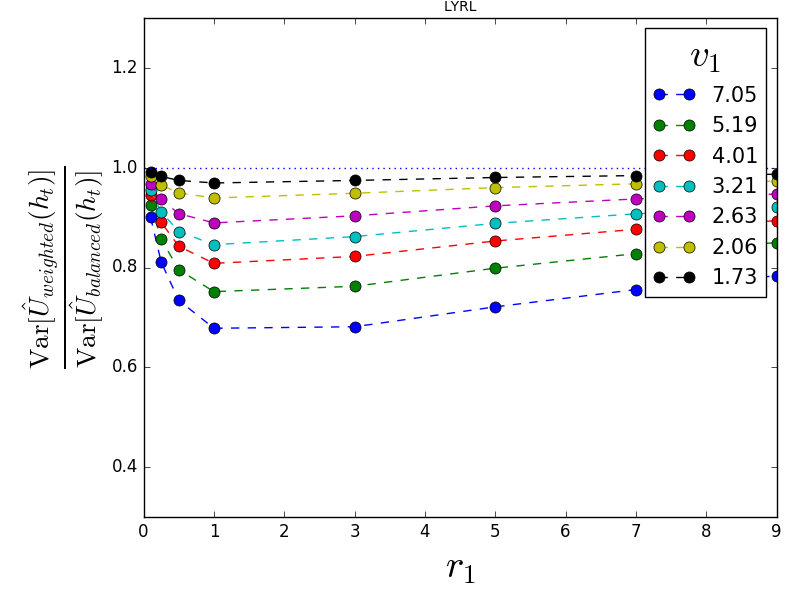}
    \caption{Variance of the Weighted IPS Estimator relative to the variance of the Balanced IPS Estimator for different $\hlog{1}$ as the relative sample size changes.
    The Weighted IPS Estimator does better than the Balanced IPS Estimator when the two logging policies differ significantly. However, the Balanced IPS Estimator performs better when the two policies are similar.}
    \label{fig:optr3r2}
\end{figure*}

\subsection{Can dropping data lower the variance of \texorpdfstring{$\Unaive$}{U-naive}?}

While we saw that dropping data improved the variance of the Naive IPS Estimator in the toy example, we first verify that this issue also surfaces outside of carefully constructed toy examples.
To this effect, Figure~\ref{fig:drop} plots the variance of the Naive IPS Estimator $\Unaive$ that uses data only from $\hlog{2}$ relative to the variance of $\Unaive$ when using data from both $\hlog{1}$ and $\hlog{2}$. The x-axis varies the relative amount of data coming from $\hlog{1}$ and $\hlog{2}$. Each solid circle on the plot corresponds to a training fraction choice for $\hlog{1}$ and a log-data-size ratio $r_1$. A lot-data-size ratio of 0 means that no data from $\hlog{1}$ is used, i.e., all data from $\hlog{1}$ is dropped. The relative divergence $v_1$ is higher when $\hlog{1}$ is trained on a lower fraction of training data since in that case $\hlog{1}$ differs more from $\hlog{2}$. A solid circle below the baseline at $1$ indicates that dropping data improves the variance in that case.

Overall, the experiments confirm that the Naive IPS Estimator shows substantial inefficiency. We observe that for high $v_1$ and small $r_1$, dropping data from $\hlog{1}$ can reduce the variance substantially for a wide range of realistic CRF policies. As $v_1$ decreases and $r_1$ increases, dropping data becomes less beneficial, ultimately becoming worse than the using all the data. This concurs with the intuition that dropping a relatively small number of high variance data samples can help utilize the low variance data samples.

\begin{figure*}[t]
    \centering
    \includegraphics[width=0.33\textwidth]{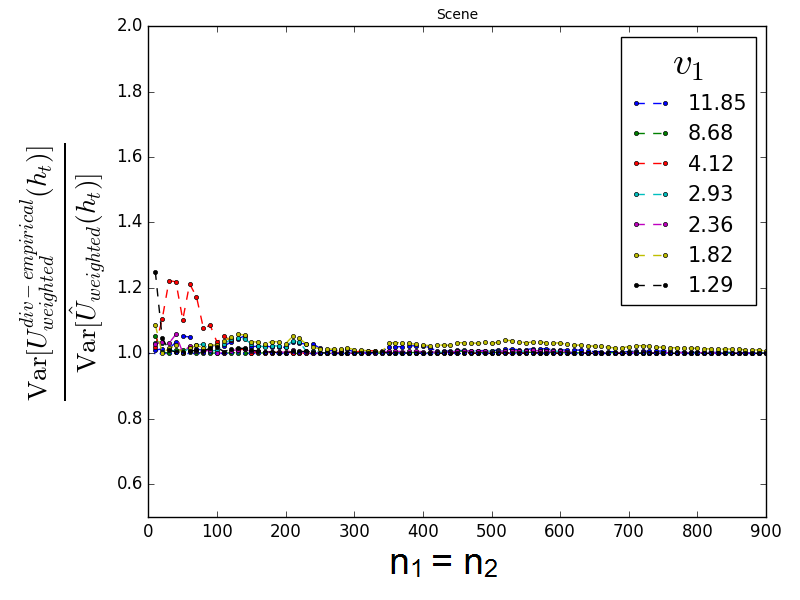} \hfill \includegraphics[width=0.33\textwidth]{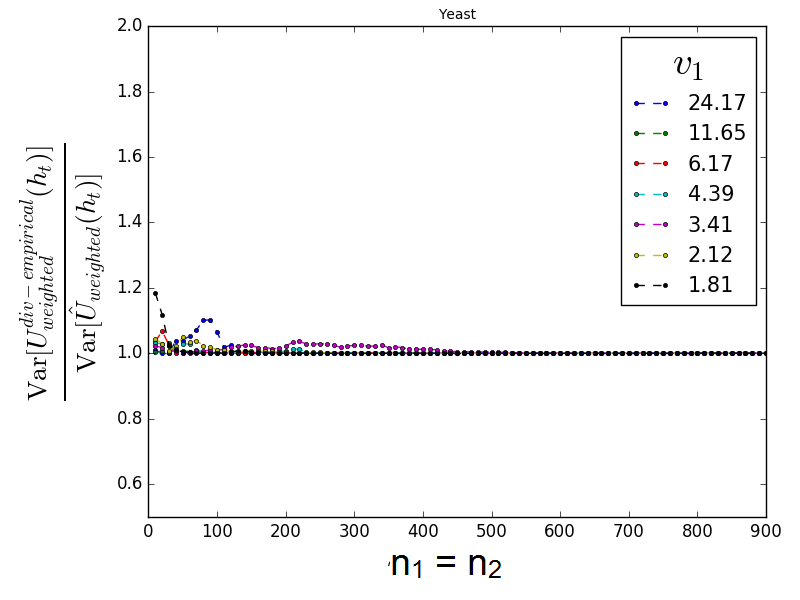} \hfill \includegraphics[width=0.33\textwidth]{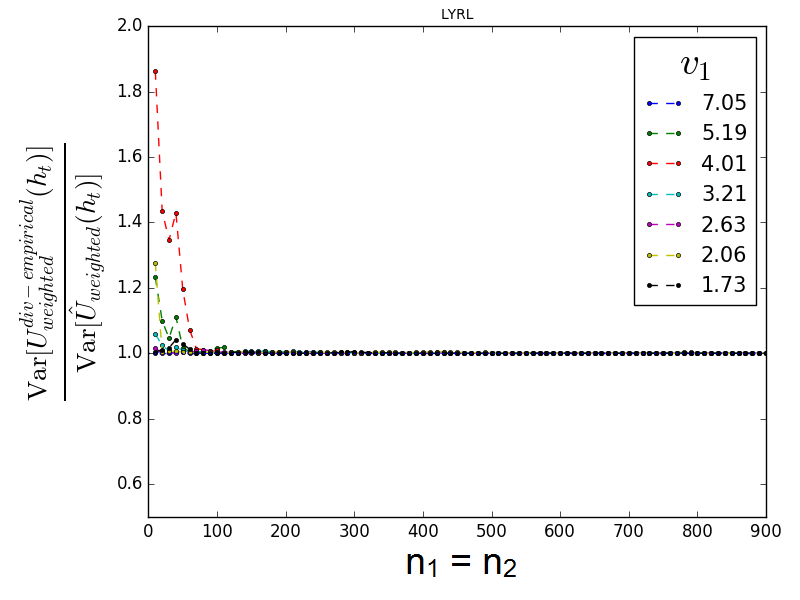}
    \caption{Variance with weights estimated from empirical divergences relative to optimal weights for the Weighted IPS Estimator. The estimation works very well when there is sufficient amount of log data. We chose $m_1 = m_2$, i.e. $r_1 = 1$ for convenience. Similar trends were observed for other values of $r_1$.}
    \label{fig:div}
\end{figure*}

%\begin{figure*}[t]
 %   \centering
  %  \includegraphics{scene-estr3} \hfill \includegraphics{yeast-estr3} \hfill \includegraphics{rcv1_topics-estr3}
  %  \caption{Variance with weights estimated by setting unkown utility function to 1 relative to optimal weights for the Weighted IPS Estimator. Our simple estimation procedure of setting the unknown utility function to a constant works surprisingly well for the first two data-sets.}
   % \label{fig:estr3}
%\end{figure*}

\subsection{How does \texorpdfstring{$\Ubal$}{U-bal} compare with \texorpdfstring{$\Unaive$?}{U-naive}}
\label{sec:ubal}

We proved that the Balanced IPS Estimator has smaller (or equal) variance than the Naive IPS Estimator. The experiments reported in Figure~\ref{fig:r2} show the magnitude of variance reduction for $\Ubal$. In particular, Figure~\ref{fig:r2} reports the variance of the Balanced IPS Estimator relative to the variance of the Naive IPS Estimator for different logging policies $\hlog{1}$ and different data set imbalances. In all cases, $\Ubal$ performs at least as well as $\Unaive$ and the variance reduction increases when the two policies differ more (i.e. $v_1$ is large). The variance reduction due to $\Ubal$ decreases as the relative size of the log data from $\hlog{1}$ increases.   

\subsection{How does \texorpdfstring{$\Uwt$}{Uwt} compare with \texorpdfstring{$\Unaive$}{U-naive}?}

We know that the Weighted IPS Estimator always has lower variance (or equal) than the Naive IPS Estimator. The results in Figure~\ref{fig:optr3} show the magnitude of the relative variance improvement for the Weighted IPS Estimator. As in the case of the Balanced  IPS Estimator, ${ \Uwt }$ performs better than $\Unaive$ especially when the two logging policies differ substantially. This confirms the theoretical characterization of $\Uwt$ from Section~\ref{sec:varred}, where we computed the variance reduction given $r_1$ and $v_1$. The empirical findings are as expected by the theory and show a substantial improvement in this realistic setting. However, note that these experiments do not yet address the question of how to estimate the weights in practice, which we come back to in Section~\ref{sec:weightest}.  

\subsection{How does \texorpdfstring{$\Uwt$}{Uwt} compare with \texorpdfstring{$\Ubal$}{U-bal}?}

We did not find theoretical arguments whether $\Uwt$ is uniformly better than $\Ubal$ or vice versa. The empirical results in Figure~\ref{fig:optr3r2} confirm that either estimator can be preferable in some situations. Specifically, $\Uwt$ performs better when the difference between the two logging policies is large, whereas $\Ubal$ performs better when they are closer. This is an interesting phenomenon that merits future investigation. In particular, one might be able to combine the strengths of $\Uwt$ and $\Ubal$ to get a weighted form of the $\Ubal$ estimator. Since we know from the toy example that even $\Ubal$ can have lower variance with dropping data, it is plausible that it could improve if the samples were weighted non-uniformly.

\subsection{How can we estimate the weights for \texorpdfstring{$\Uwt$}{Uwt}?}
\label{sec:weightest}

We derived the optimal weights $\wtopt{i}$ in terms of $\diverg{\htar}{\hlog{i}}$. Computing the divergence exactly requires access to the utility function $\delta(x,y)$ on the entire domain $\mathcal{X} \times \mathcal{Y}$. However, $\delta(x,y)$ is known only at the samples collected as bandit feedback. We propose the following  strategy to estimate the weights in this situation. 

Each divergence can be estimated by using the empirical variance of the importance-weighted utility values available in the log data $\D{i}$. 
\begin{equation*}
    \diverghat{\htar}{\hlog{i}} = \hat{\mbox{Var}}_\D{i}\left[\frac{\del{i}{j} \cdot \htar(\y{i}{j}|\x{i}{j})}{\p{i}{j}}\right]
\end{equation*}
Under mild conditions, this provides a consistent estimate since $\x{i}{j} \sim \Pr(\Xset)$ and  $\y{i}{j} \sim \hlog{i}(\Yset|\x{i}{j})$. The weights $\wt{i}$ are then obtained using the estimated divergences.  

We tested this method by generating bandit data using the Supervised $\mapsto$ Bandit conversion method described in Section~\ref{sec:setup} for each logging policy, and then computing the weights as described above. Figure~\ref{fig:div} compares the variance of the weighted estimator with the estimated weights against the variance with the optimal weights. The x-axis varies the size of the log data for both logging policies $\hlog{1}$ and $\hlog{2}$ which are kept equal (i.e. $\size{1} = \size{2}$) for simplicity. As shown, the variance of the estimator with the estimated weights converges to that of the optimal weighted estimator within a few hundred samples for all choices of logging policies and across the three data-sets. Similar trends were observed for other values of relative log data size $r_1$ as well.  

Note that in this method we take the empirical variance of the importance-weighted utility values over each log $\D{i}$ individually to get reliable unbiased estimates of the true divergences. In contrast, the Naive IPS Estimator takes the empirical mean of the same values over the combined data $\Dall$. Therefore, the former estimation does not suffer from the suboptimality in variance that occurs due to naively combining data from different logging policies.  

Therefore, we conclude that the above method of estimating the weights performs quite well and seems well suited for practical applications.

\section{Conclusion}

We investigated the problem of estimating the performance of a new policy using data from multiple logging policies in a contextual bandit setting. This problem is highly relevant for practical applications since it reflects how logged contextual bandit feedback is available in online systems that are frequently updated (e.g. search engines, ad placement systems, product recommenders). We proposed two estimators for this problem which are provably unbiased and have lower variance than the Naive IPS Estimator. We empirically demonstrated that both can substantially reduce variance across a range of evaluation scenarios.

The findings raise interesting questions for future work. First, it is plausible that similar estimators and advantages also exist for other partial-information data settings \cite{Joachimsetal2017} beyond contextual bandit feedback. Second, while this paper only considered the problem of evaluating a fixed new policy $\htar$, it would be interesting to use the new estimators also for learning. In particular, they could be used to replace the Naive IPS Estimator when learning from bandit feedback via Counterfactual Risk Minimization \cite{swaminathan2015counterfactual}.

\begin{acks}
%The work is  supported by the \grantsponsor{GS501100001809}{National Natural
%    Science Foundation of
%    China}{http://dx.doi.org/10.13039/501100001809} under Grant
%  No.:~\grantnum{GS501100001809}{61273304}
%  and~\grantnum[http://www.nnsf.cn/youngscientsts]{GS501100001809}{Young
%    Scientsts' Support Program}.
    
This work was supported in by under NSF awards IIS-1615706 and IIS-1513692, and through a gift from Bloomberg.
This material is based upon work supported by the National Science Foundation Graduate Research Fellowship Program under Grant No. DGE-1650441.
Any opinions, findings, and conclusions or recommendations expressed in this material are those of the author(s) and do not necessarily reflect the views of the National Science Foundation.

\end{acks}

\bibliographystyle{ACM-Reference-Format}
\bibliography{main} 

\end{document}